\documentclass[11pt]{article}

\usepackage{amsthm}
\usepackage{amsmath}
\usepackage{hyperref}
\usepackage{amsmath, amsfonts}
\usepackage{amscd}
\usepackage{amssymb}
\usepackage{amsxtra}
\newtheorem{thm}{Theorem}[section]
\newtheorem{lem}[thm]{Lemma}

\newcommand{\be}{\begin{equation}}
\newcommand{\ee}{\end{equation}}
\newcommand{\bd}{\begin{displaymath}}
\newcommand{\ed}{\end{displaymath}}

\newcommand{\secref}[1]{Section \ref{#1}}
\newcommand{\lemref}[1]{Lemma \ref{#1}}
\newcommand{\thmref}[1]{Theorem \ref{#1}}

\long\def\symbolfootnote[#1]#2{\begingroup%
\def\thefootnote{\fnsymbol{footnote}}\footnote[#1]{#2}\endgroup}

\newcommand{\bfE}{\textbf{E}}

\newcommand{\Prob}{\Pr}

\newcommand{\bfB}{\textbf{B}}
\newcommand{\event}{\mathfrak{E}}

\newcommand{\junk}[1]{}

\begin{document}
\title{Exact phase transition of backtrack-free search with implications on
the power of greedy algorithms}
\author{Liang Li$^*$ 
\and Tian Liu\thanks{Key Laboratory of High Confidence Software
Technologies (Peking University), Ministry of Education, CHINA, and
Institute of Software, School of Electronic Engineering and Computer
Science, Peking University, Beijing 100871, China. Email:
powercuo@pku.edu.cn (for Liang Li), lt@pku.edu.cn (for Tian Liu)}
\and Ke Xu\thanks{National Lab of Software Development Environment,
School of Computers, Beihang University, Beijing 100083, China.
Email: kexu@nlsde.buaa.edu.cn}
    }
\maketitle
\begin{abstract}

Backtracking is a basic strategy to solve constraint satisfaction
problems (CSPs). A satisfiable CSP instance is backtrack-free if a
solution can be found without encountering any dead-end during a
backtracking search, implying that the instance is easy to solve. We
prove an exact phase transition of backtrack-free search in some
random CSPs, namely in Model RB and in Model RD. This is the first
time an exact phase transition of backtrack-free search can be
identified on some random CSPs. Our technical results also have
interesting implications on the power of greedy algorithms, on the
width of superlinear dense random hypergraphs and on the exact
satisfiability threshold of random CSPs.

\end{abstract}

\newpage

\section{Introduction}

In constraint satisfaction problems (CSPs), values are assigned to
variables to fulfil constraints among these variables
\cite{D1992b,RvBW2006}.
Backtracking is a basic strategy to solve CSPs
\cite{DP1960,DLL1962,GB1965,BR1975,F1978,K1968}. A CSP instance is
called \emph{backtrack-free}, if we can always extend from scratch a
partial assignment to a solution without any reassignment (or
backtracking) along a linear ordering on variables, and at each
variable we only need to keep the extended partial assignment
compatible with these constraints among assigned variables, implying
that the instance is easy to solve \cite{F1982}.
In practice, backtrack-freeness is a very
desirable property in many applications
\cite{DAGJ1995,JHD1997,BCFR2004,H+2004,S+2004,BP2004}.
In theory, sufficient conditions and on random instances for
backtrack-freeness have been studied
\cite{F1982,D1992,vB1994,DAGJ1995,vBD1997,PG1997,JHD1997,KV2000,S2001,DFM2003}.
Here, we study backtrack-freeness from a theoretical point of views
along these two lines.

The sufficient conditions for backtrack-freeness on CSPs were given
by Freuder in terms of strong consistency and the width of
constraint graph \cite{F1982,F1985,F1988}, by van Beek and Dechter
in terms of local and global consistency \cite{D1992,vB1994},
constraint tightness and looseness \cite{vBD1997}, by Dakic et al in
terms of overlap of cliques in interval graph
representation\cite{DAGJ1995}, by Jackson et al in terms of
k-consistency and overlap in constraint graphs \cite{JHD1997}, by
Pang and Goodwin in terms of $\omega$-consistency and
tree-structured $\omega$-graph associated with constraint
hypergraphs \cite{PG1997}, and by Kolaitis and Vardi in terms of
$k$-locality \cite{KV2000}. Here, yet another sufficient condition
in terms of what we call \emph{vertex-centered consistency} and the
width of constraint hypergraph is given.

A non-zero probability of backtrack-freeness on random instances for
a range of parameter values was used by Smith to lower bound the
satisfiability threshold \cite{S2001}. Dyer, Frieze and Molloy
obtained a threshold for backtrack-freeness with respect to the
parameter of the domain size of binary CSPs with a linear number of
constraints \cite{DFM2003}. Here we identify an \emph{exact}
threshold of backtrack-freeness with respect to the density
parameter for non-binary CSPs
with a superlinear number of
constraints. This is the first time an exact phase transition of
backtrack-freeness can be identified on random CSPs. Before, the
exact phase transition results of algorithmic behaviors are rare and
mainly about resolution \cite{ABM2004,MS2007}.

Our proofs work by first showing a phase transition result about
variable-centered consistency and then
estimating the width of a random hypergraph by determining the
existence of specific $k$-cores. As far as we know, this is the
first $k$-core result on $k$-uniform hypergraphs with $rn\ln n$
hyperedges and $n$ vertices. In our case, the width increases
\emph{smoothly} with the density parameter, in sharp contrast to the
earlier $k$-core threshold results in literatures for sparse
hypergraphs
\cite{B1984,L1991,PSW1996,DFM2003,C2004,M2005,FR2003,FR2004,JL2006a,JL2006b,CW2006,R2008,
K2008a}.

Our results have implications on the power of greedy algorithms,
since below the backtrack-freeness threshold we can find a solution
in a \emph{greedy} manner for almost all instances, while above the
threshold we are forced to search with backtracking for almost all
instances, even for satisfiable instances. To this end, we define
the width of greedy algorithms.
Also, our results show that for Model RB/RD, the satisfiability
threshold and some local property threshold are linked tightly, so
we suggest that a similar link might exist for random $3$-SAT.

This paper is organized as follow. In \secref{sec:pre} we fix our
notations and give all necessary definitions and some known results.
In \secref{sec:csp} we show the exact phase transition of
backtrack-freeness. In \secref{sec:graph} we show results about
width and $k$-cores in random hypergraphs. In \secref{sec:open} we
discuss some implications of our results.

\section{Preliminaries} \label{sec:pre}

In constraint satisfaction problems (CSPs), a set of variables
$\{u_1, u_2, \cdots, u_n\}$ and a set of constraints
$\{C_1,C_2,\cdots,C_m\}$ are given for each instance. We call $n$
the input size and ratio $\frac{m}{n}$ the constraint density. Each
variable can take a value from a finite domain $\{1,2,\cdots,d\}$.
We allow $d$ to increase with $n$, say $d=n^\alpha$, where $\alpha$
is a constant. An assignment is a mapping from the variable set to
the domain and a partial assignment is a mapping from a variable
subset to the domain. Each constraint involves a subset of variables
and labels each partial assignment on these variables either as
compatible or incompatible, but not both. In so called $k$-CSPs,
each constraint involves $k$ variables. $2$-CSPs are also called
binary CSPs. An assignment compatible with all constraints is called
a solution. Instances with at least one solution are called
satisfiable, otherwise unsatisfiable.

In random CSPs, constraints are generated by a random process with a
small number of control parameters, leading to a probabilistic
distribution on all instances. In Model RB, given $n$ variables each
with domain $\{1,2,...,d\}$, where $d=n^\alpha$ and $\alpha>0$ is
constant, select with repetition $m=rn\ln n$ random constraints, for
each constraint select without repetition $k$ of $n$ variables,
where $k=2,3,4,...$, and select uniformly at random without
repetition $(1-p)d^k$ compatible assignments for these $k$
variables, where $0<p<1$ is constant. If in the last step above,
each assignment for the $k$ variables is selected with probability
$1-p$ as compatible independently, then it is called Model RD
(\cite{XL2000}). Model RB is asymptotically similar to Model RD just
as $G(n,M)$ is to $G(n,p)$, all asymptotic results should hold both
for Model RB/RD ( \cite{XL2000,XL2006} ). For simplicity, here we
only give proofs valid for Model RD and omit more complicated
calculations for Model RB. For Model RB/RD, not only exact
satisfiability thresholds can be identified \cite{XL2000} but also
the existence of many hard instances around the thresholds can be demonstrated 
both theoretically \cite{XL2006} and experimentally \cite{XBHL2007}.

\begin{thm} \label{thm:satisfiability-r} (\cite{XL2000}, Theorem 1)
Let $r_{cr}=-\frac{\alpha}{\ln (1-p)}$, where $\alpha>\frac{1}{k}$,
$0<p<1$ are constants and $k\ge \frac{1}{1-p}$. Then for a random
instance $\phi$ in Model RB/RD,
$$ \lim_{n \to \infty}\Prob(\phi \mbox{ is satisfiable }) =
\begin{cases}
 1   &  r<r_{cr},\\
 0   &  r>r_{cr}.
\end{cases}
$$
\end{thm}

\begin{thm} \label{thm:resolution} (\cite{XL2006}, Theorem 3)
Almost all instances in Model RB/RD have no tree-like resolutions of
length less than $2^{\Omega(n)}$ and no general resolutions of
length less than $2^{\Omega(n/d)}$.
\end{thm}

In graph theory, a hypergraph consists of some nodes and some
hyperedges. Each hyperedge is a subset of nodes. A hypergraph is
$k$-\emph{uniform} if every hyperedge contains exact $k$ nodes.
Every CSP has an underlying \emph{constraint (multi-)hypergraph}:
each variable corresponds to a node and each constraint corresponds
to a hyperedge in a natural way. The constraint hypergraphs of
random CSPs are random hypergraphs \cite{JLR2000}. The constraint
hypergraph of Model RB/RD, denoted by $HG(n,rn\ln n,k)$, is a random
$k$-uniform multi-hypergraph with $n$ nodes and $rn\ln n$
hyperedges, where $r$ is constant and $k=2,3,4,...$. Denote by $HG$
a random hypergraph from $HG(n,rn\ln n,k)$.

Let $\phi$ be an instance of CSPs. Let $u$ be a variable. Let $C$ be
a constraint involving $u$, where $C$ is called a $u$-constraint.
For any $u$, the total number of $u$-constraints is called the
\emph{degree} of $u$ and denoted as $deg(u)$. Let $C_u$ be a set of
$u$-constraints, where $C_u$ is called $u$-\emph{centered}. Denote
by $N_{C_u}$ the set of all variables involved in constraints in
$C_u$. Denote by $C_{\setminus u}$ the set of all constraints among
variables in $N_{C_u}\setminus \{u\}$. Denote by $T_{C_{\setminus
u}}$ the set of all partial assignments each compatible with all
constraints in $C_{\setminus u}$. Let $c$ be a partial assignment in
$T_{C_{\setminus u}}$. Let $v$ be a value to $u$. Denote by $c'$ the
partial assignment extending $c$ just with $u=v$.

Let $\pi$ be a linear ordering on variables in $\phi$, say
$u_1<u_2<\cdots < u_n$. Denote by $C_{u_i}^\pi$ the set of all
$u_i$-constraints such that all constraints in $C_{u_i}^\pi$ are
among $\{u_1,u_2,\cdots,u_{i}\}$. The \emph{width} of $u_i$ under
$\pi$ is just $|C_{u_i}^\pi|$. The width of $\pi$ is $\max_{i}
width(u_i)$, denoted by $width(\pi)$. The width of $\phi$ is
$\min_{\pi} width({\pi})$, denoted by $width(\phi)$. For constraint
hypergraphs, the degree and width can be defined in a similar way.
The width and the associated optimal linear ordering can be found
efficiently \cite{F1982,F1985,F1988,MT1996}. Moreover, the
\emph{linkage} of a hypergraph $HG$ is the minimum degree of all its
nodes, denoted by $linkage(HG)$. A $k$-\emph{core} of a hypergraph
is a nonempty maximal subgraph with minimum degree $k$. In
\cite{F1982}, it was essentially proved that the width of a
hypergraph is equal to the maximal linkage of its subgraphs.

Consider the following strategy to solve $\phi$. At step $1$, we put
an arbitrary value to $u_1$. Assume that after step $i-1$, we have a
partial assignment $c$ on $\{u_1,u_2,\cdots,u_{i-1}\}$ which is
compatible with all constraints among $\{u_1,u_2,\cdots,u_{i-1}\}$.
At step $i$, we find a value $v$ for $u_i$ such that, when $c$ is
extended with $u_i=v$, the resulting assignment $c'$ is compatible
to all constraints among $\{u_1,u_2,\cdots,u_{i}\}$. Such a $v$ is
called \emph{available}. When there are more than one available
$v$'s, we take an arbitrary one from them. Note that the only
requirement to $v$ is that, when $c$ is extended with $u_i=v$, the
resulting assignment $c'$ is compatible with all constraints among
$\{u_1,u_2,\cdots,u_{i}\}$. In fact, the only requirement for $v$ is
that $c'$ is compatible with all constraints in $C_{u_i}^\pi$. If at
each step $i$ ($1\le i \le n$), for every partial assignment $c$, we
can always find such a value $v$ for $u_i$, then we say that $\phi$
is backtrack-free under $\pi$. Otherwise, we say that $\phi$ is not
backtrack-free under $\pi$. If there is a $\pi$ such that $\phi$ is
backtrack-free under $\pi$, then we say that $\phi$ is
\emph{backtrack-free}.

If whenever $|C_u|\le t$ , then for every $c\in T_{C_{\setminus
u}}$, we can \emph{always} find a $v$ such that $c'$ is compatible
with all constraints in $C_u$ (that is, for all $C\in C_u$, $c'$ is
compatible with $C$), then we say that $u$ is
\emph{variable-centered $t$-consistent}. If every $u$ in an instance
is variable-centered $t$-consistent, then we call this instance
variable-centered $t$-consistent and $t$ is called the
\emph{critical size} of this variable-centered consistency.





Denote by $\bfE(X)$ the expectation of a random variable $X$,
 $\bfB(n,p)$ the binomial distribution,
 $\Prob(\event)$ the probability of event $\event$.
 An event $\event$ occurs \emph{with high probability}, or \textbf{whp},
 if $\lim_{n \to \infty}\Prob(\event)=1$.

\begin{lem}(\emph{Chernoff Bound})\cite{ASE2002,MR1996,JLR2000,MU2005}
 For a random variable $X$ with distribution $\bfB(n,\frac{\mu}{n})$ and
 $0<\epsilon<1$, we have
$\Prob(X\leq (1-\epsilon)\mu)\leq e^{-\mu\epsilon^2/2}$ and
$\Prob(X\geq(1+\epsilon)\mu)\leq e^{-\mu\epsilon^2/3}$,
 and for any $\mu_h > \mu,$
 $\Prob(X\geq(1+\epsilon)\mu_h)\leq e^{-\mu_h\epsilon^2/3}.$
 \end{lem}

Finally, $f\ll g$ means $f=o(g)$ or $\lim_{n \to \infty}
\frac{f}{g}=0$. A useful inequality is $1-x<e^{-x}<1-x+o(x)$ for
small $x>0$.

\section{The exact threshold of backtrack-freeness} \label{sec:csp}

In this section we give the exact threshold of backtrack-freeness
for the Model RB and Model RD. We first give a sufficient condition
for backtrack-freeness.

\textbf{Note}: In this section, when we use $\phi$, $\pi$,
$C_u^{\pi}$, $u$, $C_u$, $N_{C_u}$, $C_{\setminus u}$,
$T_{C_{\setminus u}}$, $c$, $v$, $c'$ and $C$, we implicitly assume
that they adhere to the descriptions in \secref{sec:pre}.

\begin{thm} \label{thm:condition}
If $\phi$ is vertex-centered $width(\phi)$-consistent, then $\phi$
is backtrack-free.
\end{thm}
\begin{proof}
By definition of backtrack-freeness, clearly
$$
\phi \mbox{ is backtrack-free }\Leftrightarrow \exists \pi, \forall
u, \forall c\in T_{C^{\pi}_{\setminus u}}, \exists v, \forall C\in
C^{\pi}_u, c' \mbox{ is compatible with C}.
$$
By definition of width, there is a $\pi$ such that
$width(\phi)=width(\pi)$. Under $\pi$, for all $u_i$, $width(u_i)\le
width(\pi)=width(\phi)$. Then the vertex-centered
$width(\phi)$-consistency guarantees that at each $u_i$, the partial
assignment can be extended as desired by backtrack-free search.
\end{proof}

As a warm up, we upper bound the number of $u$-constraints  for any
$u$ as $O(\ln n)$.

\begin{lem}\label{lem:degree}
$\max_{u} deg(u)<(1+\sqrt{\frac{6}{kr}})kr\ln n$ \textbf{whp}.
\end{lem}
\begin{proof}
Since the total number of constraints is $rn\ln n$, every constraint
involves exactly $k$ vertices, and a given vertex appears in a
constraint with probability $\frac{k}{n}$, $deg(u)$ is a random
variable with binomial distribution $\bfB(rn\ln n, \frac{k}{n})$. By
Chernoff bound, for any $u$ we have $\Prob(deg(u)\geq ( 1 +
\sqrt{\frac{6}{kr}})kr\ln n )\leq \frac{1}{n^2}$. By Union bound, we
have $\Prob(\exists u, deg(u)\geq (1+\sqrt{\frac{6}{kr}})kr\ln n)
\le n\cdot \frac{1}{n^2}=\frac{1}{n}$, so $\Prob(\forall u,
deg(u)<(1+\sqrt{\frac{6}{kr}})kr\ln n) \ge 1-\frac{1}{n}$, that is,
$\max_{u} deg(u)<(1+\sqrt{\frac{6}{kr}})kr\ln n$ \textbf{whp}.
\end{proof}

Our main observation is that there is a threshold for density
parameter $r$ in Model RB/RD, such that below this threshold, almost
all instances are variable-centered consistent for some critical
size, while above this threshold, almost all instance are not
variable-centered consistent for another critical size. Happily, the
two critical sizes can be very close!

\begin{lem} \label{lem:consistent}
Let $r_{bf}=-\frac{\alpha}{k\ln (1-p)}$, where $\alpha>0$, $0<p<1$,
$k=2,3,4,...$ are constants. If $r<r_{bf}$,  $0 < \epsilon <
\min(\frac{r_{bf}-r}{r},\frac{1}{2})$ and $t=(1+\epsilon)kr\ln n$,
then $\Prob(\forall u, u$ is vertex-centered $t$-consistent $) \ge
1-e^{-n^{O(1)}}$.
\end{lem}
\begin{proof}
Given $u$, $C_u$, $c$, $v$, $C$ and $c'$ as described in
\secref{sec:pre} and only consider $C_u$'s with $|C_u|\le t$,
$$
u \mbox { is vertex-centered } t\mbox{-consistent} \Leftrightarrow
\forall C_u, \forall c, \exists v, \forall C, c' \mbox{ is
compatible with } C.
$$
Under the distribution on random instances of Model RD, we have
\begin{eqnarray*}
\Prob( c' \mbox{ is compatible with } C) &=& 1-p,\\
\Prob(\forall C, c' \mbox{ is compatible with } C) &=& (1-p)^{|C_u|},\\
\Prob(\exists C, c' \mbox{ is incompatible with } C) &=& 1-(1-p)^{|C_u|},\\
\Prob(\forall v, \exists C, c' \mbox{ is incompatible with } C) &=&
(1-(1-p)^{|C_u|})^d.
\end{eqnarray*}
To apply the Union bound on $u$, $C_u$ and $c$, we only need to
upper bound $(1-(1-p)^{|C_u|})^d$ and the number of choices of $u$,
$C_u$ and $c$ respectively. To upper bound $(1-(1-p)^{|C_u|})^d$,
recall that $\epsilon < \frac{r_{bf}-r}{r}$, denote
$\delta=r_{bf}-(1+\epsilon)r>0$ and $\gamma=-\delta{k}\ln(1-p)>0$,
then $ |C_u|\le t=(1+\epsilon)kr\ln n = (r_{bf}-\delta)k\ln n =
(-\frac{\alpha}{\ln(1-p)}-\delta k)\ln n = -\frac{\alpha
-\gamma}{\ln(1-p)}\ln n$, so we have $ (1-(1-p)^{|C_u|})^d \le
(1-(1-p)^{-\frac{\alpha -\gamma}{\ln(1-p)}\ln n})^{n^\alpha}
=(1-n^{-\alpha+\gamma})^{n^\alpha} \le
(e^{-n^{-\alpha+\gamma}})^{n^\alpha} =
e^{-n^{\gamma}}=e^{-n^{O(1)}}$, the last inequality is by
$1-x<e^{-x}$ for $x\ne 0$. The number of possible choices of $u$ is
no greater than $n=e^{\ln n}$. By lemma $\ref{lem:degree}$, for any
$u$, the total number of $u$-constraints is $deg(u)=O(\ln n)$
\textbf{whp}, so the number of possible choices of $C_u$ is no more
than $2^{deg(u)}=e^{O(\ln n)}$ \textbf{whp}. For any $C_u$, the
number of variables in $N_{C_u}$ is no more than $k|C_u|$, since
each constraint includes exactly $k$ variables. Each variable can
take at most $d=n^\alpha$ different values, so the number of
possible choice of $c$ is $|T_{C_{\setminus u}}|\le
d^{|N_{C_u}\setminus \{u\}|}\le
d^{|N_{C_u}|}\le(n^\alpha)^{k|C_u|}\le n^{kt}=n^{O(\ln n)}=e^{O((\ln
n)^2)}$. By Union bound, we have $ \Prob(\exists u, \exists C_u,
\exists c, \forall v, \exists C, c' \mbox{ is incompatible with } C)
\le e^{\ln n} \cdot e^{O(\ln n)} \cdot e^{O((\ln n)^2)}\cdot
e^{-n^{O(1)}} = e^{-n^{O(1)}}$. By taking complement, we have $
\Prob(u \mbox { is vertex-centered } t\mbox{-consistent}) =
\Prob(\forall u, \forall C_u, \forall c, \exists v, \forall C, c'
\mbox{ is compatible with } C) \ge 1-e^{-n^{O(1)}}$.
\end{proof}

\begin{lem} \label{lem:not consistent}
Let $r_{bf}=-\frac{\alpha}{k\ln (1-p)}$, where $\alpha>0$, $0<p<1$,
$k=2,3,4,...$ are constants. If  $r>r_{bf}$, $0 < \epsilon <
\min(\frac{r-{r_{bf}}}{r},\frac{1}{2})$,
$\delta=(1-\epsilon)r-r_{bf}>0$, $\gamma=-\delta{k}\ln(1-p)>0$ and
$t=(1-\epsilon)kr\ln n$, then for all $u$ and for all $C_u$ with
$|C_u|\ge t$, $\Prob(\forall c, \exists v, \forall C, c' \mbox{ is
compatible with C}) < n^{-\gamma n^{\Omega(\ln n)}}$.
\end{lem}

\begin{proof}
As in proof of \lemref{lem:consistent} but only consider $C_u$'s
with $|C_u|\ge t$,
\begin{eqnarray*}
\Prob(\forall v, \exists C, c' \mbox{ is incompatible with } C) &=&
(1-(1-p)^{|C_u|})^d,\\
\Prob(\exists v, \forall C, c' \mbox{ is compatible with } C) &=&
1-(1-(1-p)^{|C_u|})^d,\\
\Prob(\forall c, \exists v, \forall C, c' \mbox{ is compatible with
} C) &=&(1-(1-(1-p)^{|C_u|})^d)^{|T_{C_{\setminus u}}|}.
\end{eqnarray*}
This time we only need to lower bound $(1-(1-p)^{|C_u|})^d$ and
$|T_{C_{\setminus u}}|$. To lower bound $(1-(1-p)^{|C_u|})^d$,
recall that $\epsilon < \frac{r-r_{bf}}{r}$,
$\delta=(1-\epsilon)r-r_{bf}>0$ and $\gamma=-\delta{k}\ln(1-p)>0$,
then $ |C_u|\ge t=(1-\epsilon)kr\ln n = (\delta+r_{bf})k\ln n =
(\delta k-\frac{\alpha}{\ln(1-p)})\ln n = \frac{-\alpha
-\gamma}{\ln(1-p)}\ln n$, so $(1-(1-p)^{|C_u|})^d \ge
(1-(1-p)^{\frac{-\alpha -\gamma}{\ln(1-p)}\ln n})^{n^\alpha}
=(1-n^{-\alpha-\gamma})^{n^\alpha} \approx
 e^{-n^{-\gamma}}$, the last
approximation is by $(1-\frac{1}{n})^n\approx \frac{1}{e}$. To lower
bound $|T_{C_{\setminus u}}|$, recall that $C_{\setminus u}$ denote
the set of all constraints among variables in
$N_{C_u}\setminus\{u\}$ and
$$
\bfE(|T_{C_{\setminus
u}}|)=(1-p)^{|C_{\setminus u}|}\cdot d^{|N_{C_u}\setminus
\{u\}|}=(1-p)^{|C_{\setminus u}|} \cdot d^{|N_{C_u}|-1},
$$
so we only need to upper bound $|C_{\setminus u}|$ and to lower
bound $|N_{C_u}|$.

To upper bound  $|C_{\setminus u}|$, we only need to upper bound
$|N_{C_u}|$, since each constraint in $|C_{\setminus u}|$ is among
variables in $N_{C_u}\setminus \{u\}$. In turn, we only need to
upper bound $|C_u|$, since each variable in $N_{C_u}$ is contained
in some constraint in $C_u$ and each constraint contains exactly $k$
variables. By \lemref{lem:degree}, $|C_u|=O(\ln n)$ \textbf{whp}, so
$|N_{C_u}|\le k|C_{u}|=O(\ln n)$ \textbf{whp}. Since each constraint
contains exactly $k$ variables, the probability that a given
constraint is among $N_{C_u}\setminus \{u\}$ is $
\frac{\binom{|N_{C_u}|-1}{k} }{\binom{n}{k}}\le
\frac{\binom{|N_{C_u}|}{k} }{\binom{n}{k}}\le
(\frac{|N_{C_u}|}{n})^k = (\frac{O(\ln n)}{n})^k$. Since the total
number of constraints is $rn\ln n=O(n\ln n)$, we have
$\bfE(|C_{\setminus u}|) \le (\frac{O(\ln n)}{n})^k\cdot O(n\ln n)
=\frac{O((\ln n))^2}{n^{k-1}}=o(1)$ for $k\ge 2$. By Markov
inequality, $ \Prob(|C_{\setminus u}|\ge 1)\le \bfE(|C_{\setminus
u}|)=o(1)$, so $|C_{\setminus u}|=0$ \textbf{whp}.

To lower bound $|N_{C_u}|$, the number of variables involved in
constraints in $C_u$, we only need to upper bound the probability
that a variable does not appear in any constraint in $C_u$. Since
each constraint includes exactly $k$ variables, a variable appears
in a constraint with probability $\frac{k}{n}$, not appears in a
constraint with probability $1-\frac{k}{n}$,  and not appears in all
constraints in $C_u$ with probability
$(1-\frac{k}{n})^{|C_u|}<(e^{-\frac{k}{n}})^t=e^{-\frac{kt}{n}}<1-\frac{kt}{n}+o(\frac{kt}{n})$,
using $1-x<e^{-x}<1-x+o(x)$ for $x\ne 0$ and $|C_u|\ge t$. So
$\bfE{(|N_{C_u}|)} = n[1-(1-\frac{k}{n})^{|C_u|}]> n\cdot(
\frac{kt}{n}-o(\frac{kt}{n}))=kt-o(\ln n)$, since $t=O(\ln n)$. By
Chernoff bound, $\Prob( |N_{C_u}| \leq (1-\epsilon)kt) = o(1)$, so
$|N_{C_u}| > (1-\epsilon)kt$ \textbf{whp}.

Now we have
$$
\bfE(|T_{C_{\setminus u}}|)=(1-p)^{|C_{\setminus u}|}
d^{|N_{C_u}|-1}\ge (1-p)^0\cdot
(n^{\alpha})^{(1-\epsilon)kt-1}=n^{\Omega(\ln n)} \mbox{\textbf{
whp}}.
$$
By the second moment method similar to that in \cite{XL2000}, we can
prove that $ |T_{C_{\setminus u}}|\ge n^{\Omega(\ln n)}$ \textbf{
whp}. So $\Prob(\forall c, \exists v, \forall C, c' \mbox{ is
compatible with C})=(1-(1-(1-p)^{|C_u|})^d)^{|T_{C_{\setminus
u}}|}<(1-e^{-n^{-\gamma}})^{|T_{C_{\setminus u}}|}<(n^{-\gamma})
^{n^{\Omega(\ln n})}=n^{-\gamma n^{\Omega(\ln n)}}. $
\end{proof}

Finally, we can prove the exact phase transition of
backtrack-freeness on Model RB/RD.

\begin{thm} \label{thm:backtrack-free-r}
Let $r_{bf}=-\frac{\alpha}{k\ln (1-p)}$, where $\alpha>0$, $0<p<1$,
$k=2,3,4,...$ are constants. Then
$$ \lim_{n \to \infty}\Prob(\phi \mbox{ is backtrack-free }) =
\begin{cases}
 1   &  r<r_{bf},\\
 0   &  r>r_{bf}.
 \end{cases}
$$
\end{thm}

\begin{proof}
If $r<r_{bf}$, let
$0<\epsilon<\min(\frac{r_{bf}-r}{r},\frac{1}{2})$. From
\lemref{lem:consistent}, $\phi$ is vertex-centered
$(1+\epsilon)kr\ln n$-consistent \textbf{whp}. From
\lemref{lem:width upper bound}, $width(\phi)<(1+\epsilon)kr\ln n$
\textbf{whp}. By definition, for $t'<t$, vertex-centered
$t$-consistency implies vertex-centered $t'$-consistency, so $\phi$
is vertex-centered $width(\phi)$-consistent \textbf{whp}. By
\thmref{thm:condition}, $\phi$ is backtrack-free \textbf{whp}. This
completes the first half of our proof.

If $r>r_{bf}$, let $\epsilon <
\min(\frac{r-r_{bf}}{r},\frac{1}{2})$. By \lemref{lem:width lower
bound}, for any $\pi$, $width(\pi)\ge (1-\epsilon)kr\ln n$
\textbf{whp}, so exists a $u$ such that $|C_u^{\pi}|\ge
(1-\epsilon)kr\ln n$. By \lemref{lem:not consistent}, for any $u$,
$$
\Prob(\forall c\in T_{C^{\pi}_{\setminus u}}, \exists v, \forall
C\in C^{\pi}_u, c' \mbox{ is compatible with }C) = n^{-\gamma
n^{\Omega(\ln n)}}.
$$
Since the number of choices of $\pi$ is $n!$, by Union bound,
\begin{eqnarray*}
  \Prob(\phi \mbox{ is backtrack-free })
&\le& n!\Prob(\forall u, \forall c\in T_{C^{\pi}_{\setminus u}},
\exists v,
 \forall C\in C^{\pi}_u, c' \mbox{ is compatible with }C)\\
&\le& n!\Prob(\forall c\in T_{C^{\pi}_{\setminus u}}, \exists v,
 \forall C\in C^{\pi}_u, c' \mbox{ is compatible with }C)\\
&\le& n!\cdot n^{-\gamma n^{\Omega(\ln n)}} \approx
(\frac{n}{e})^n\cdot n^{-\gamma n^{\Omega(\ln n)}}=o(1).
\end{eqnarray*}
 This completes our proof.
\end{proof}



\section{Width of random hypergraphs} \label{sec:graph}

In this section we determine the width of some random hypergraphs
with a superlinear number of hyperedges. We apply a probabilistic
method mainly inspired by \cite{DFM2003,M2005} to detect the
existence of $k$-cores. Denote by $HG$ a random hypergraph from
$HG(n,rn\ln n,k)$. We show that \textbf{whp} the width of $HG$,
denoted as $width(HG)$, is asymptotically equal to average degree
$kr\ln n$, due to high concentration of distribution of node degree
in $HG$.


\begin{lem} \label{lem:width upper bound}
For any $0<\epsilon<1$, $width(HG) \leq (1+\epsilon)kr\ln n$
\textbf{whp}.
\end{lem}
\begin{proof}
 The number of
hyperedges in a subgraph $G'\subseteq HG$ is a random variable
$X_{G'}$. If $G'$ has $f(n)$ nodes, when adding a hyperedge to $HG$
with repetition, the value of $X_{G'}$ increases by $1$ with
probability $\frac{{f(n)\choose k}}{{n \choose k}}$, so $X_{G'}$
distributes as $\bfB(rn\ln n,\frac{{f(n)\choose k}}{{n \choose
k}})$, and
\begin{equation} \label{eq:EXG}
\bfE(X_{G'}) = rn\ln n\cdot\frac{{f(n)\choose k}}{{n \choose k}}\leq
r\ln n\cdot f(n)<(1+\epsilon)r\ln n\cdot f(n).
\end{equation}

Let $avd(G')$ denote the average degree of $G'$. By (\ref{eq:EXG})
and Chernoff Bound, we have
\begin{equation} \label{AvdChernoff}
\Prob(avd(G')>(1+\epsilon)kr\ln n)=\Prob(X_{G'}>(1+\epsilon) r\ln
n\cdot f(n))\leq e^{-r\ln n\cdot f(n)\cdot
\epsilon^2/3}=n^{-r\epsilon^2/3\cdot f(n)}.
\end{equation}

Let random variable $N_i=|\{G'| \mbox{ subgraph } G' \mbox{ has } i
\mbox{ nodes} \wedge avd(G')>(1+\epsilon)kr\ln n \geq 1\}|$ and $N =
N_1+N_2+...+N_n$. Since the width of a hypergraph is equal to the
maximal linkage of its subgraphs \cite{F1982}, we have
$$\Prob(width(HG)>(1+\epsilon)kr\ln n) =\Prob(\exists G'\subseteq
HG,linkage(G')>(1+\epsilon)kr\ln n)$$
\begin{equation}\label{eq:EN}
\leq \Prob(\exists G'\subseteq HG,avd(G')>(1+\epsilon)kr\ln n) \leq
\Prob( N_1+N_2+...+N_n \geq 1)\leq \bfE(N).
\end{equation}

Below we show that $\bfE(N)$ tends to $0$ by showing that
$\bfE(N_{f(n)}) = o(1/n)$.

\textbf{Case 1}. When $f(n)$ is large, namely $n^{1-r\epsilon^2/3}
\ll f(n)\leq n$, since by (\ref{AvdChernoff}), we have
$$\bfE(N_{f(n)}) \leq {n \choose
f(n)}\cdot n^{-r\epsilon^2/3\cdot f(n)}
\leq(\frac{en}{f(n)})^{f(n)}\cdot n^{-r\epsilon^2/3\cdot f(n)}
=(\frac{en^{1-r\epsilon^2/3}}{f(n)})^{f(n)}= o(1/n).$$

\textbf{Case 2}. When $f(n)$ is small, that is $f(n) \ll n$, since
by (\ref{eq:EXG}),  for all $i>(1+\epsilon)r\ln n\cdot f(n)$, we
have $\Prob(avd(G')=i) \leq \Prob(avd(G')=(1+\epsilon)kr\ln n)$, so
 $$\Prob(avd(G')>(1+\epsilon)kr\ln n)
 \leq n\Prob(avd(G')=(1+\epsilon)kr\ln n)
=n\Prob(X_{G'}=(1+\epsilon)r\ln n\cdot f(n))$$
$$\leq n{rn\ln n \choose (1+\epsilon)r\ln n\cdot f(n)}(\frac{{f(n)
\choose k}}{{n \choose k}})^{(1+\epsilon)r\ln n\cdot f(n)} \leq
n{rn\ln n \choose (1+\epsilon)r\ln n}
(\frac{f(n)}{n})^{k(1+\epsilon)r\ln n\cdot f(n)}$$
$$\leq n(\frac{ern\ln n}{(1+\epsilon)r\ln n\cdot f(n)})^{(1+\epsilon)r\ln n\cdot
f(n)}\cdot(\frac{f(n)}{n})^{k(1+\epsilon)r\ln n\cdot f(n)}
=n(C_1\cdot\frac{f(n)}{n})^{C_2f(n)\ln n},$$ where $C_1>0$ and
$C_2>0$ are two constants. Then,
$$\bfE(N_{f(n)}) \leq {n \choose f(n)}n(C_1\cdot\frac{f(n)}{n})^{C_2f(n)\ln
n} \leq
(\frac{en}{f(n)})^{f(n)}n(C_1\cdot\frac{f(n)}{n})^{C_2f(n)\ln n}$$
$$\leq(C'_1\cdot\frac{f(n)}{n})^{C'_2f(n)\ln n}= o(1/n),$$ where
$C'_1>0$ and $C'_2>0$ are two constants.

The above two cases already overlap each other, so we
 can upper bound $\bfE(N)$ as
 \begin{equation}\label{eq:ENo}
 \bfE(N) \leq \sum_{f(n)\ll n}\bfE(N_{f(n)}) + \sum_{f(n)\gg
n^{1-r\epsilon^2/3}}\bfE(N_{f(n)} \leq 2n\cdot o(1/n)= o(1).
\end{equation}
 The lemma follows from (\ref{eq:EN}) and (\ref{eq:ENo}).
\end{proof}

\begin{lem} \label{lem:width lower bound}
For any $0<\epsilon<1$, $width(HG) \geq (1-\epsilon)kr\ln n$
\textbf{whp}.
\end{lem}
\begin{proof}
Let $m=(1-\epsilon)kr\ln n$. Since the width of a hypergraph is
equal to the maximal linkage of its subgraphs \cite{F1982},
we need to prove the
existence of a subgraph of $HG$ whose minimum degree is at least $m$
\textbf{whp}, or the existence of an $m$-core \textbf{whp}, which
can be achieved by an analysis of the following standard $m$-core
detecting algorithm: while there exists any node with degree less
than $m$, randomly select such a node and delete it together with
all hyperedges containing it, if there is no node left then output
No, otherwise output the remaining subgraph.


Let $X_i$ denotes the number of nodes whose degree are less than $m$
after deleting the $i$th node. Let $W_{i,j} = \{u | u \textrm{ has
degree } j \textrm{ after deleting the } i \text{th node} \}$, then
$X_i = |W_{i,1}| + |W_{i,2}| + ... + |W_{i,m-1}|$. Obviously, an
$m$-core exists if and only if the node-hyperedge deletion process
cannot delete all nodes, and if and only if there exists a $j<n$,
such that $X_j = 0$. Since
$$\Prob(width(HG)\geq m )=
 \Prob( \exists j<n, X_j=0)
 \geq \Prob( X_0 + |W_{0,m}| < n^{\delta} \wedge \exists j<n, X_j
=0) $$
\begin{equation}\label{eq:Decomp}
 = \Prob(X_0 + |W_{0,m}| < n^{\delta})
 \cdot \Prob(\exists j<n, X_j=0\mid X_0 + |W_{0,m}|<n^{\delta}),
\end{equation}
where $\delta \in (0,1)$ will be determined later, we only need to
estimate the last two probabilities.

Whenever we add a hyperedge to $HG$ with repetition, a node's degree
increases by $1$ with a probability of $k/n$. So the degree of each
node in $HG$ is a random variable with distribution $\bfB(rn\ln
n,k/n)$. By Chernoff bound, for a specific node $u$, we have
\begin{equation}
\Prob(u \textrm{'s degree is not more than } m) \leq
n^{-kr\epsilon^2}.\nonumber
\end{equation}
So $\bfE(X_0 + |W_{0,m}|) \leq n\cdot n^{-kr\epsilon^2/2}
=n^{1-kr\epsilon^2/2}$. Then by Markov inequality, we have
$\Prob(X_0 + |W_{0,m}| \geq n^{\delta}) \leq \bfE(X_0 +
|W_{0,m}|)/n^{\delta} \leq n^{1-kr\epsilon^2/2-\delta}$, so for
$\delta \in (1-kr\epsilon^2/2,1)$, we have
\begin{equation} \label{eq:PartOne}
\Prob(X_0 + |W_{0,m}| < n^{\delta})=1-\Prob(X_0 + |W_{0,m}| \geq
n^{\delta})  \geq 1-o(1).
\end{equation}

Now assume that $X_0 + |W_{0,m}| < n^{\delta}$, where
$1-kr\epsilon^2/2 < \delta < 1$. When deleting the $(i+1)$th node,
at most $(m-1)$ hyperedges are deleted together, which contain at
most $(m-1)(k-1)$ other nodes, among which only the $m$-degree nodes
will count for $X_{j+1}$. Since any subhypergraph with a given
degree sequence is uniformly random, see for example \cite{JLR2000},
such a subhypergraph can be generated according to the configuration
model \cite{JLR2000}, so the probability that one deleted hyperedge
containing an $m$-degree node is
$$q_i = m|W_{i,m}|/\sum_{j\geq 1}{j|W_{i,j}|}.$$
Let $T_i$ be a random variable with distribution
$\bfB((m-1)(k-1),q_i)$, then the sequence of random variables
$X_0,X_1,...$ can be discribed as
\begin{equation}
X_0<n^{\delta} \mbox{ and } X_{i+1}\leq X_i -1 + T_i.\nonumber
\end{equation}
Since $|W_{0,m}|\leq X_0 + |W_{0,m}| < n^{\delta}$ and $\sum_{j\geq
1}{j|W_{0,j}|} = krn\ln n$, we have
$$(m-1)(k-1)q_0 < ((1-\epsilon)kr\ln n-1)(k-1)
\frac{(1-\epsilon)kr\ln n\cdot n^{\delta}}{krn\ln n} =o(1).$$ After
deleting the $i$th node, comparing with the beginning of the
node-hyperedge deletion process, the number of $m$-degree node
increases by at most $(m-1)(k-1)i$, and the sum $\sum_{j\geq
1}{j|W_{i,j}|}$ decreases by at most $(m-1)i$. So for all
$i<n^{\delta'}$, where $\delta'\in(\delta,1)$, we have
$$(m-1)(k-1)q_i <
\frac{(m-1)(k-1)m(|W_{0,m}|+(m-1)(k-1)n^{\delta'})}{krn\ln
n-(m-1)n^{\delta'}}$$ $$<
\frac{(m-1)(k-1)m(n^{\delta}+(m-1)(k-1)n^{\delta'})}{krn\ln
n-(m-1)n^{\delta'}}= o(1).$$ Thus, $\bfE(T_i)=(m-1)(k-1)q_i$ can be
 arbitrary small. Without loss of generality, let $q$ be
determined by $(m-1)(k-1)q = 1/2$. Let $D_i$ be a random variable
with distribution $\bfB((m-1)(k-1),q)$. We now define a new sequence
of random variables $Y_0,Y_1,...$ by
$$Y_0 = n^{\delta} \mbox{ and } Y_{i+1} = Y_i - 1 + D_i.$$
Clearly, for all $i<n^{\delta'}$, $X_i$ is statistically dominated
by $Y_i$, and $\sum_{i=1}^{n^{\delta'}}{D_i}$ distributes as
$\bfB(n^{\delta'}(m-1)(k-1),q)$. Therefore,
$$\Prob(\exists j<n, X_j = 0\mid X_0 + |W_{0,m}| < n^{\delta})
\geq \Prob(\exists j<n, X_j = 0\mid X_0 < n^{\delta})$$ $$\geq
\Prob(\exists j<n^{\delta'}, Y_j = 0\mid Y_0 = n^{\delta}) \geq
\Prob Y_{n^{\delta'}} < 0) = \Prob(\sum_{i=1}^{n^{\delta'}}{D_i} <
n^{\delta'}-n^{\delta})$$
\begin{equation}\label{eq:PartTwo}
 = 1-\Prob(\sum_{i=1}^{n^{\delta'}}{D_i}\geq
n^{\delta'}-n^{\delta}) \geq 1-\Prob(\sum_{i=1}^{n^{\delta'}}{D_i}
\geq 2/3n^{\delta'}) \geq 1-exp(-\frac{1}{54}n^{\delta'}) = 1-o(1),
\end{equation}
the last second step above is by Chernoff bound. The lemma follows
from (\ref{eq:Decomp}),(\ref{eq:PartOne}) and (\ref{eq:PartTwo}).
\end{proof}


\section{Discussions}\label{sec:open}

We have proved that in some random CSP models (Model RB/RD), the
backtrack-freeness threshold $r_{bf}$ in
\thmref{thm:backtrack-free-r} not only exists, but also has a fixed
ratio to the satisfiability threshold $r_{cr}$ in
\thmref{thm:satisfiability-r}, that is, $r_{bf}=\frac{r_{cr}}{k}$,
where $k$ is the number of variables in each constraints.

The first implications is on the power of greedy algorithms. A CSP
algorithm is called \emph{greedy}, if at each step we choose an
unassigned variable by some rule and assign an available value for
it, here by \emph{availability} we mean that the extended partial
assignment is compatible with all constraints among all assigned
variables. The availability is a natural feature in common greedy
algorithms. A greedy algorithm succeeds on an instance if all
variables can be assigned in this way, fails otherwise. To specify a
greedy algorithm, we need to specify the rule to choose the next
variable from unassigned variables and the rule to choose an
available value for the variable. In turn, every greedy algorithm
specifies a linear ordering, called \emph{induced ordering}, on all
variables in an instance, and the width of the induced ordering on
constraint graph can be called the \emph{width} of the greedy
algorithm on this instance. Note that some greedy algorithms have a
fixed linear ordering not depending on instances thus a fixed width.
For others, we can define the width of the greedy algorithm as the
maximum width over all instances.

 If an instance is backtrack-free under an ordering
$\pi$, then every greedy algorithm as described above with induced
ordering $\pi$ will succeeds on this instance, no matter how to
choose an available value for each variable. Moreover, if an
instance is vertex-centered $t$-consistent, then every greedy
algorithm as described above with induced width no greater than $t$
will succeed on this instance, no matter how to choose an available
value for each variable. As far as we know, this is the first time
to define explicitly the width of a greedy CSP algorithm and relate
it to the power of greedy algorithms on CSPs.

As a concrete example to the above discussion, let us consider Model
RB/RD. On the one hand, Model RB/RD is $NP$-complete for all
positive values for the density parameter $r$.

On the other hand, at least in a constant portion to the satisfiable
range of values for parameter $r$ (that is,
$r<r_{bf}=\frac{r_{cr}}{k}$), there is an easily determined ordering
of variables such that almost surely, every greedy algorithm
following that ordering will succeed on almost all instances of
Model RB/RD, in sharp contrast to its worst-case complexity. When
$k=2$, at least in half portion to the satisfiable range of values
for parameter $r$ (that is, $r<r_{bf}=\frac{r_{cr}}{2}$), almost all
instances can be easily solved by greedy algorithms. While for
instances above $r_{bf}$, with high probability, there does not
exist such an ordering to guarantee the success of every greedy
algorithm. This implies that the exact threshold of
backtrack-freeness obtained in this paper can also be viewed as a
threshold for the power of greedy algorithms.

The second implication is about the satisfiability threshold for
random CSPs. For Model RB/RD, the exact threshold of
\emph{satisfiability} is
 $r_{cr}=-\frac{\alpha}{\ln (1-p)}$ (Theorem 1 in \cite{XL2000}),
 which is independent of $k$, the number of variables in each constraint,
 while the exact threshold of \emph{backtrack-freeness}  is
$r_{bf}=-\frac{\alpha}{k\ln (1-p)}=\frac{r_{cr}}{k}$, which
decreases with $k$. For fixed $k$, these two thresholds have a fixed
ratio $k$, so an exact link between them exists. Note that the
backtrack-freeness threshold also coincides with the threshold of
vertex-centered consistency, a local property. So our results show
an evidence that for random CSPs, the exact threshold of
satisfiability might has links to thresholds of some local
properties, say local consistency. Based on this evidence, we
propose the following two steps to attack the notorious problem of
determining the satisfiability threshold for random $3$-SAT.
\begin{itemize}
\item Step 1: reduce the satisfiability threshold to some local
property (say local consistency) threshold.
\item Step 2: determine the local property threshold.
\end{itemize}
Since reductions are commonly used in computer science and local
properties are usually easier to handle than global properties,
hopefully the two steps each will be easier than directly attacking the
original satisfiability threshold problem.

\section*{Acknowledgement}
We thank Professor Mike Molloy for helpful comments on an earlier
version of this paper.

\end{document}